\documentclass[journal]{IEEEtran}

\usepackage{booktabs}
\usepackage{amsthm}
\theoremstyle{plain}
\usepackage[compress]{cite}
\usepackage{graphics}
\usepackage{epstopdf}
\usepackage{epsfig}
\usepackage{lineno,hyperref}
\usepackage{array}
\usepackage{amsmath,amssymb}
\usepackage{mathrsfs}
\usepackage{subfigure} 
\usepackage{multirow}
\usepackage{xcolor}
\usepackage{bm}
\usepackage{eqparbox}
\usepackage{algorithm}  
\usepackage{algorithmic}
\usepackage{graphicx}
\newtheorem{myTheo}{Theorem}
\newtheorem{myLemma}{Lemma}
\newtheorem{myRemark}{Remark}
\newtheorem{myCorollary}{Corollary}
\newtheorem{myDefinition}{Definition}
\usepackage{amsfonts}
\usepackage{enumerate}
\usepackage{indentfirst}
\usepackage{ragged2e}
\usepackage{url}

\allowdisplaybreaks[4]

\begin{document}

\title{Enhanced Principal Component Analysis under A Collaborative-Robust Framework}
\author{Rui Zhang, \IEEEmembership{~Member~IEEE}, Hongyuan Zhang, and Xuelong Li$^*$, \IEEEmembership{~Fellow~IEEE}
\thanks{$*$ Corresponding author.}

\thanks{Rui Zhang, Hongyuan Zhang, and Xuelong Li are with School of Computer Science
and School of Artificial Intelligence, Optics and Electronics (iOPEN), Northwestern Polytechnical University, Xi'an 710072, Shaanxi, P. R. China.}

\thanks{E-mail: \{ruizhang8633, hyzhang98\}@gmail.com, xuelong\_li@nwpu.edu.cn.}

}

\maketitle


\begin{abstract}
Principal component analysis (PCA) frequently suffers from the disturbance 
of outliers and thus a spectrum of robust extensions and variations of PCA 
have been developed. However, existing extensions of PCA treat all 
samples equally even those with large noise. 
In this paper, we first introduce a general collaborative-robust weight learning 
framework that combines weight learning and robust loss in a non-trivial way. 
More significantly, under the proposed framework, only a part of well-fitting 
samples are activated which indicates more importance during training, and others, 
whose errors are large, will not be ignored. 
In particular, the negative effects of inactivated samples are alleviated 
by the robust loss function. 
Then we furthermore develop an enhanced PCA which adopts a point-wise $\sigma$-loss 
function that interpolates between $\ell_{2,1}$-norm and squared Frobenius-norm and meanwhile 
retains the rotational invariance property. 
Extensive experiments are conducted on occluded datasets from two aspects including 
reconstructed errors and clustering accuracy. 
The experimental results prove the superiority and effectiveness of our model.
\end{abstract}


\section{introduction}\label{sec:introduction}

Principal component analysis (\textit{PCA}) \cite{pca,l1pca-nongreedy,pca-16} 
as an unsupervised approach has been widely applied in diverse fields such as 
astronomy, biology, economics, and computer vision. 
The intention of conventional PCA can be stated from different perspectives. 
On the one hand, PCA can be interpreted as to project raw data into 
low-dimensional space which remains the largest variance of transformed data. 
Therefore, PCA is conventionally viewed as an unsupervised learning algorithm 
of dimensionality reduction. 
Since the dimensionality of data in many real applications could be extremely 
high, learning algorithms frequently suffer from the dimension curse. 
Consequently, PCA, owing to no need for supervised information, is often 
employed to reduce dimension during the preprocessing phase. 
On the other hand, the motivation of PCA is to find a low-rank approximation 
of data points. In other words, PCA intends to find a subspace with the 
definite rank whose error is as small as possible. 
According to this, PCA is also employed to compress data and reconstruct data points from 
polluted ones. 
In particular, plenty of 2DPCA \cite{2dpca-04,2dpca-05,2dpca-12,2dpca-capped} 
models have been proposed to process spatial data like images.

In practice, datasets always contain a quantity of noises and outliers. Unfortunately, the classical PCA is sensitive to outliers, which leads to poor performance.
To address this problem, a spectrum of robust PCA algorithms \cite{rpca,l1pca-nongreedy,pcal1,r1pca,l1pca,2dpca-capped,rspca,rpca-om} have been developed by applying different loss functions. 
For instance, \cite{pcal1} simply replaces squared $\ell_2$-norm with $\ell_{1}$-norm. 
R1PCA \cite{r1pca} utilizes the robust $\ell_{2,1}$-norm, \textit{i.e.}, non-squared $\ell_2$-norm, and thus retains the rotational invariance as well. 
L1PCA \cite{l1pca} modifies the objective and applies $\ell_1$-norm such that it has rotational invariance. 
Furthermore, the algorithm proposed in \cite{2dpca-capped} promotes the robustness via capped models \cite{capped-model-15,capped-model-17}. 
The robust PCA with optimal mean (PCA-OM) \cite{rpca-om} takes the mean into account as well as the projection matrix, and therefore gains impressive performance. 

However, the existing PCA methods still have several drawbacks. 
Firstly, the robustness of their losses is limited. 
Specifically speaking, the $\ell_{2,1}$-norm is sensitive to small noises and the capped models \cite{2dpca-capped} intensively depend on the choice of $\epsilon$ which varies dramatically on different datasets and is hard to set an appropriate value. 
Secondly, the highly polluted samples and well-fitting samples are treated equally which may do harm to the final result. 
More precisely, the existing algorithms only consider how to alleviate the errors caused by outliers but not improve the impact of well-fitting samples. 
As a result, the mentioned limitations hinder the performance of the existing robust PCA techniques.

This paper is a deep extension of our conference paper \cite{PCA-AN} which 
intends to reduce dimension via detecting clean and occluded samples 
and learn their importance automatically.
Considering the motivations of all important variations of PCA and our previous work,
we further propose a novel model, Enhanced Principal Component Analysis (\textit{EPCA}), 
in this paper. 
The motivation of EPCA is shown in Fig. \ref{figure_framework}. 
The merits of this paper are listed as follows:
\begin{itemize}
    \item [1.] We proposed a novel framework named as collaborative-robust weight learning which utilizes the robust loss to reduce the errors caused by outliers and novel weight learning to enhance the effect of well-fitting data simultaneously. In particular, only $k$ best-fitting samples are activated. Therefore, the robustness of models is enhanced intensively.
    \item [2.] A novel robust loss, $\sigma$-loss, is designed to address the limitation of loss functions in existing PCA approaches. Moreover, the point-wise $\sigma$-loss is further developed to adapt the collaborative-robust framework.
    \item [3.] Under the framework, a novel model, Enhanced Principal Component Analysis (\textit{EPCA}), is proposed. EPCA can learn weights for all samples to augment the effect of well-fitting samples which cooperates with the robust loss to promote the robustness of our model.
\end{itemize}

\textbf{Notations}: In this paper, all uppercase words denote matrices 
while vectors are written in bold lowercases. 
For a vector $\bm a$, $\|\bm a\|_1$ and $\|\bm a\|_2$ denote the $\ell_1$ 
and $\ell_2$-norm respectively. 
For a matrix $A$, the $i$-th column is represented as 
$\bm a_i$,
the $(i, j)$-th entry is $a_{ij}$, 
and the transpose is denoted by $M^T$.
$\|A\|_F^2 = \sum_{i,j} a_{ij}^2$ represents the squared Frobenius-norm 
while $\|a\|_{2,1} = \sum_i \|\bm a_i\|_2$ denotes the $\ell_{2,1}$-norm. 
When $A$ is a squared matrix, ${\rm tr}(A)$ is the trace operator where 
${\rm tr}(A) = \sum_i m_{ii}$. 
$\mathcal{R}(A)$ denote the linear space spanned by its column vectors.  
$\nabla f(\bm x)$ represents the gradient of $f(\bm x)$ w.r.t. $\bm x$, 
\textit{i.e.}, $\nabla f(\bm x) = [\partial f / \partial x_2, \partial f / \partial x_2, \cdots, \partial f / \partial x_n]^T$.
To keep the organization clear, all proofs are shown in Section \ref{section_proof}.

\begin{figure} [t]
    \centering
    \includegraphics[width=0.95\linewidth]{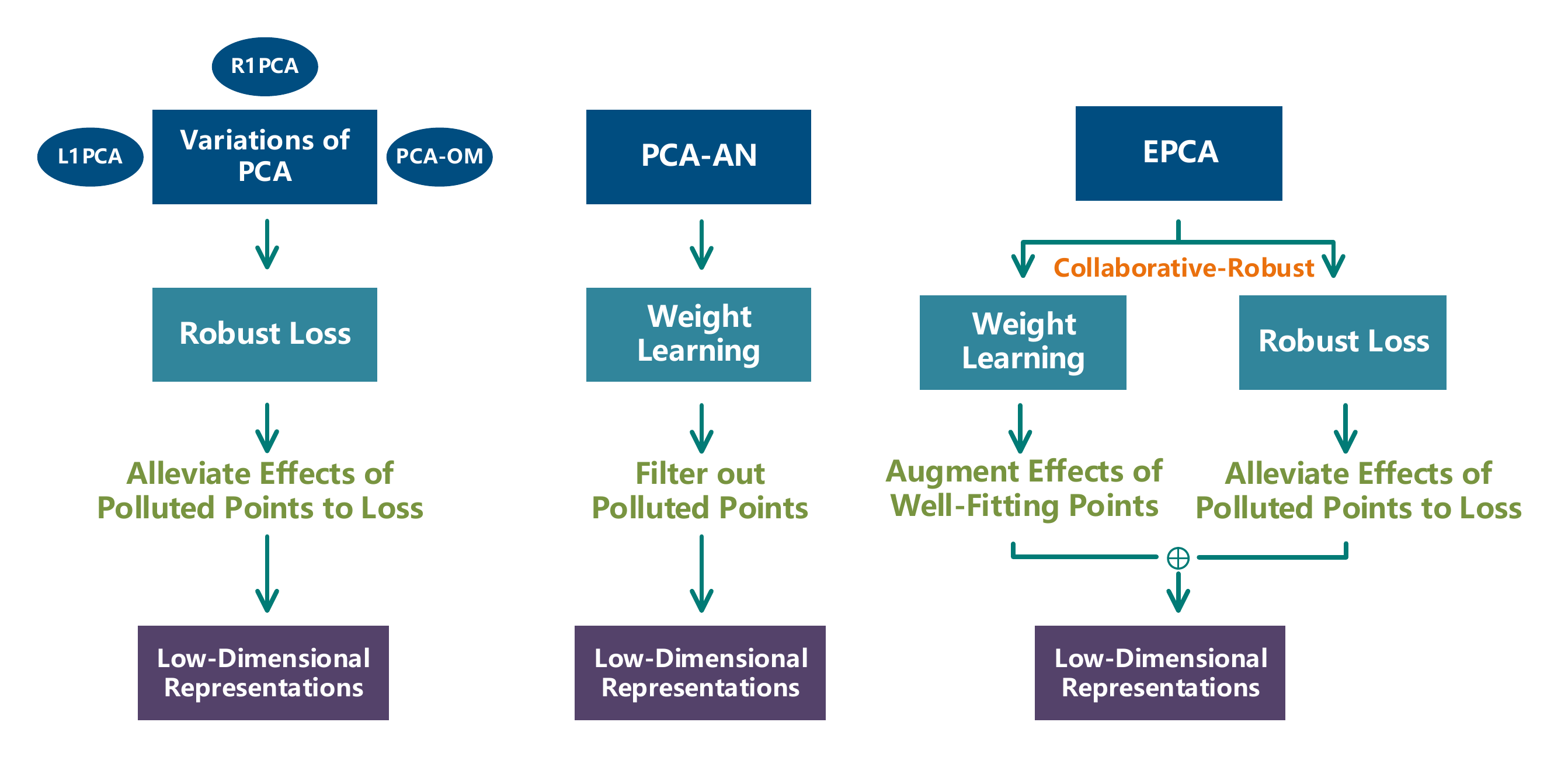}
    \caption{Core motivations of varitions of PCA, PCA-AN, and EPCA. 
    PCA-AN is our conference work.}
    \label{figure_framework}
    \vspace{-3mm}
\end{figure}

\section{Related work}
As a well-known unsupervised dimensionality reduction method, the classical 
PCA \cite{pca} attempts to project the raw features into a low-dimensional space via a 
linear transformation. Formally speaking, given a dataset $X \in \mathbb{R}^{d \times n}$
where the $i$-th column $\bm x_i \in \mathbb{R}^d$ represents the $i$-th sample,
the goal of PCA is to find a $c$-dimensional representation for $X$, which 
remains the information and structure of the original features. 
The standard procedure of PCA consists of two steps: 
    1) Preprocessing: Centralize $X$ via $X - \frac{1}{n} \sum_{i=1}^n \bm x_i \textbf{1}^T$;
    2) Eigenvalue decomposition (or singular value decomposition for acceleration): Find a linearly transformed
space $U^T X$ that maximizes the variance of $n$ samples. 
As classical PCA is sensitive to outliers and practical datasets are always 
occluded by noises more or less, plenty of methods \cite{l1pca,l1pca-nongreedy,lppca,pcal1,r1pca,rspca} 
have been proposed to promote the robustness of PCA. For instance, L1PCA \cite{l1pca,l1pca-nongreedy} 
utilizes $\ell_1$-norm to measure the errors, while L$p$PCA \cite{lppca} employs $\ell_p$-norm 
to further improve the scalability. 
Besides, the models proposed in \cite{r1pca,rspca,pca-16,rpca-om} also intend to promote 
the robustness via replacing the measurement to alleviate the impacts of polluted samples. 
Overall, all these methods get stuck in finding a robust loss to cope with noises and 
outliers such that the performance hits a bottleneck. 
They fail to exploit the importance of different data points. 
In other words, since there are only a small amount of data points that are occluded 
severely, an ideal dimensionality reduction model should be able to distinguish
clean samples and polluted samples automatically and then treat them differently
during training. Accordingly, our conference paper \cite{PCA-AN} proposes 
an extension of PCA that depends on no robust loss but introduces a point-wise 
importance learning mechanisms into PCA. The innovation brings impressive 
improvement in performance. Considering the intentions of theses existing methods,
we further propose a novel extension of PCA under a general learning framework.

\section{A Collaborative-Robust Framework for Weight Learning}
\label{section_framework}

Our previous conference paper \cite{PCA-AN} focuses on filtering out outliers automatically via a 
sparse weight learning framework. The framework detects outliers and noises according to the loss of each sample. 
More specifically, the samples with large loss will be assigned small weights, 
or even be neglected by setting the corresponding weights as 0 automatically. 
Through treating data points differently during training, the weight learning
mechanism is able to improve the robustness intensively.

The classical thought to enhance the robustness is to utilize a robust loss 
to replace the conventional $\ell_2$-norm and Frobenius-norm. 
The underlying motivation is to alleviate the impact of outliers on the 
loss function since the polluted points usually result in large losses.
An ideal robust loss should be able to avoid producing exaggerated values 
when encountering outliers. 

Intuitively, combining the two methods is a feasible way to further improve
the performance. However, \textbf{do we really need both weight learning and robust loss meanwhile?}

Since both of the two mechanisms promote the robustness by reducing the effects 
caused by occluded points, it is reduplicative to employ both of them. 
In practice, we may need only one of them to enhance the models due to the same motivation. 
Along with this insight, we extend our conference work via developing a framework,
which incorporates them elegantly. 
Under this framework, the weight learning acts on clean data while the robust 
loss focuses on polluted data, such that they work collaboratively.

\subsection{A Collaborative-Robust Framework for Weight Learning}

As we emphasize in the last subsection, the purposes of weight learning and robust loss 
are reduplicative. 
In this section, we propose a general framework which can employ both 
adaptive weights and a robust loss naturally.
All existing methods only 
concentrate on decreasing errors of polluted samples via robust loss functions 
but fail to promote the impact of well-fitting ones. 
Here we propose a novel framework, named as collaborative-robust (\textit{co-robust}) framework, that 
has the following merits:
\begin{itemize}
    \item [1)] The adaptive weight learning attempts to augment the impact of well-fitting 
    samples, \textit{i.e.}, activated samples; 
    \item [2)] the robust function intends to alleviate disturbance of samples with large 
    noises, \textit{i.e.}, inactivated samples.
\end{itemize}
\subsubsection{Formulation}
The proposed co-robust weight learning framework is formulated as 
\begin{equation}
    \label{co_robust_framework}
    \begin{split}
        \min \limits_{w_i, \bm \theta} & \sum \limits_{i=1}^n \frac{1}{1 - w_i} f(\bm x_i; \bm \theta),
        ~~ s.t.~  \sum \limits_{i=1}^n w_i = 1, 0 \leq w_i < 1.
    \end{split}
\end{equation}
where $w_i$ is the corresponding weight, $f(\bm x_i; \bm \theta)$ measures the error of the sample $\bm x_i$, $\bm \theta$ represents the parameters to learn. Formally, we can further formulate $f(\bm x_i; \bm \theta)$ as
\begin{equation}
	f(\bm x_i) = h(\bm g(\bm x_i; \bm \theta)).
\end{equation}
$\bm g(\bm x_i; \bm \theta)$ is the a vector-output, or matrix-output function that is defined to calculated the loss of $\bm x_i$. $h(\cdot)$ is the loss function.
For instance, in the least-squares regression \cite{lasso}, $\bm g(\bm x_i) = \bm \theta^T \bm x_i - y_i$ and $h(\cdot) = \|\cdot\|_2^2$ where $y_i$ denotes the target of $\bm x_i$. 
In practice, many machine learning methods aim to improve the performance by employing a robust loss function such as $\ell_1$-norm, $\ell_{2,1}$-norm, capped norm, \textit{etc}.
In the following discussion, we will use $f(\cdot)$ rather than $h(\bm g(\bm x_i; \bm \theta))$ to keep notations uncluttered if unnecessary. 

Here we give some insights about the above elegant framework. 
Let $\hat w_i = \frac{1}{1 - w_i}$ be the \textit{direct weight} of $\bm x_i$.
The direct weight is a quantity that is positive correlated with weight. 
It implies the real importance of $x_i$ to the final learner.
From the intuitive aspect, samples with less reconstruction errors will be 
assigned to larger weights to minimize the global loss. 
The direct weight, positive correlated with $w_i$, implies the real importance of $x_i$ during training of the model.
Unlike sparse models, samples whose weights are assigned to 
0 are still retained in the model since $\hat w_i = \frac{1}{1 - 0} = 1$. 
In this case, samples with $w_i = 0$, \textbf{the negative impact will be alleviated 
by the robust loss $h(\cdot)$}.
On the contrary, if $\bm x_i$ is assigned to a large weight (\textit{i.e.}, 
$w_i \rightarrow 1$), then the direct weight will approach to infinity (\textit{i.e.}, 
$\hat w_i \rightarrow \infty$) such that \textbf{the samples will be important ones to the model}.
Therefore, our model is equivalent to augment the importance of samples 
with less reconstruction errors, and alleviate the negative effects of the 
highly polluted samples meanwhile.

From a mathematical perspective, we have the following formulation via Taylor's 
theorem
\begin{equation}
    g(w) = \frac{1}{1 - w} = \sum \limits_{j=0}^\infty \frac{g^{(j)}(0)}{k!} w^j = \sum \limits_{j=0}^\infty w^j ,
\end{equation}
where $g^{(j)}(0)$ is the $k$-order derivative. Consequently, Eq. (\ref{co_robust_framework}) can be rewritten as
\begin{equation}
    \label{fusion}
    \begin{split}
        \min \limits_{w_i, \bm \theta} & \sum \limits_{j=1}^\infty \sum \limits_{i=1}^n w_i^j f(\bm x_i; \bm \theta), 
        ~~ s.t.~  \sum \limits_{i=1}^n w_i = 1, 0 \leq w_i < 1.
    \end{split}
\end{equation}
Note that every item $\sum \limits_{i=1}^n w_i^j f(\bm x_i; \bm \theta)$ is the adaptive weighted learning with the exponential regularization which is widely applied in weight learning tasks. So the above objective defined in Eq. (\ref{co_robust_framework}) could be viewed as a fusion of weighted learning with different $j$. 
The sparsity of the exponential regularized weight learning is not gauranteed. 
Surprisingly, the solution of the co-robust framework can be proved as a sparse 
solution, which will be elaborated in the next part. 
Apart from the sparsity, the designed weight learning is parameter-free, which is another
attractive property.

\subsubsection{Optimization}
Suppose $\forall i, f_i > 0$. To learn $w_i$ adaptively, we fix $x_i$ and thus the Lagrangian of problem (\ref{co_robust_framework}) is given as 
\begin{equation}
    \label{co_robust_Lagrangian}
    \mathcal L = \sum \limits_{i=1}^n \frac{f_i}{1 - w_i} + \lambda (1 - \sum \limits_{i=1}^n w_i) + \sum \limits_{i=1}^n \gamma_i(-w_i) ,
\end{equation}
where  $\lambda$, $\gamma_i$ are Lagrangian multipliers and $f_i$ represents $f(x_i)$ for simplicity. 
The KKT conditions are listed as follows
\begin{equation}
    \label{co_robust_framework_KKT}
    \left \{
    \begin{array}{c l}
            \frac{f_i}{(1-w_i)^2} - \lambda - \gamma_i &= 0 ,\\
            \gamma_i w_i &= 0 ,\\
            \gamma_i &\geq 0 ,\\
            \sum \limits_{i=1}^n w_i &= 1 .
    \end{array}
    \right.
\end{equation}
According to the third condition, we obtain
\begin{equation}
    \begin{array}{l}
        \left \{
        \begin{aligned}
            &w_i \neq 0,  \gamma_i = 0 \Rightarrow w_i = 1 - \sqrt{\frac{f_i}{\lambda}} > 0 ,\\
            &w_i = 0, \gamma_i \geq 0 \Rightarrow 1 - \sqrt{\frac{f_i}{\lambda}} \leq 0 .\\
        \end{aligned}
        \right.
    \end{array}
\end{equation}
\begin{algorithm}[t]
    \caption{Algorithm to optimize problem (\ref{co_robust_framework})}
    \label{algorithm_framework}
    \begin{algorithmic}
        \REQUIRE Arbitrary initial weights $\{w_i\}_{i=1}^n$ and variables to optimize $\{\bm x_i\}_{i=1}^n$ with proper initialization.\\
        \REPEAT
            \STATE Fix $w_i$ and update $x_i$.
            \REPEAT
                \STATE Find valid $k$ via certain strategy. \\
            \UNTIL{$k$ satisfies constraint (\ref{k_constraint}).}
            \STATE Update $w_i$ by Eq. (\ref{w_update}).
        \UNTIL{convergence}
        \ENSURE Weights $\{w_i\}_{i=1}^n$ and variables $\{x_i\}_{i=1}^n$.
    \end{algorithmic}
\end{algorithm}
Accordingly, the formulation to update $w_i$ could be written as
\begin{equation}
    \label{w_update_raw}
    w_i = (1 - \sqrt{\frac{f_i}{\lambda}})_+ .
\end{equation}
where $(x)_+ = max(0, x)$. Without loss of generality, suppose $f_1 \leq f_2 \leq \cdots \leq f_n$, 
which means $w_1 \geq w_2 \geq \cdots \geq w_n$. 
Without loss of generality, assume that
\begin{equation}
    \label{lambda_constraint}
    \sqrt{f_k} < \sqrt \lambda \leq \sqrt{f_{k+1}} ,
\end{equation}
then $w_i = 0$ for any $i > k$. Combining with the forth formula of Eq. (\ref{co_robust_framework_KKT}), we have
\begin{equation}
    \label{co_robust_framework_lambda_update}
    \sum \limits_{i=1}^n w_i = k + \sum \limits_{i=1}^k \sqrt{\frac{f_i}{\lambda}} = 1 \Rightarrow \sqrt{\lambda} = \frac{1}{k-1} \sum \limits_{i=1}^k \sqrt f_i .
\end{equation}
Substitute Eq. (\ref{co_robust_framework_lambda_update}) into Eq. (\ref{w_update_raw}) and we obtain
\begin{equation}
    \label{w_update}
    w_i = (1 - \frac{(k-1) \sqrt{f_i}}{\sum \limits_{j=1}^k \sqrt{f_j}})_+ .
\end{equation}
However, $k$ is not a hyper-parameter, \textit{i.e.}, not an arbitrary integer, since $k$ has to satisfy the following constraint
\begin{equation}
    \label{k_constraint}
    \sqrt{f_k} < \sqrt \lambda \leq \sqrt{f_{k+1}}
    \Rightarrow
    \frac{\sum \limits_{i=1}^k \sqrt{f_i}}{\sqrt{f_{k+1}}} + 1 \leq k < \frac{\sum \limits_{i=1}^k \sqrt{f_i}}{\sqrt{f_{k}}} + 1 ,
\end{equation}
which is derived from Eq. ({\ref{lambda_constraint}). 
Thus the optimal $\bm w$, whose $i$-th element is $w_i$, 
only has $k$ non-zero entries provided that $f_i > 0$ for any $i$. 
Since we assume that $f_i > 0$, $\frac{\sum _{i=1}^k \sqrt{f_i}}{\sqrt{f_{k+1}}} > 0$,
which indicates $k \geq 2$. Accordingly, our method will not get stuck in 
the common trivial solution, where only one weight is non-zero, of 
weight learning.
Furthermore, the following theorem proves that $k$ is unique in this case.

\begin{myTheo} \label{uniqueness}
    The number of activated samples, $k$, is unique if $f_i > 0$ holds where $i = 1, 2, \cdots, n$.
\end{myTheo}

\begin{figure}[t]
    \centering
    \subfigure[$\sigma=0.1$]{
        \includegraphics[width=0.46\linewidth]{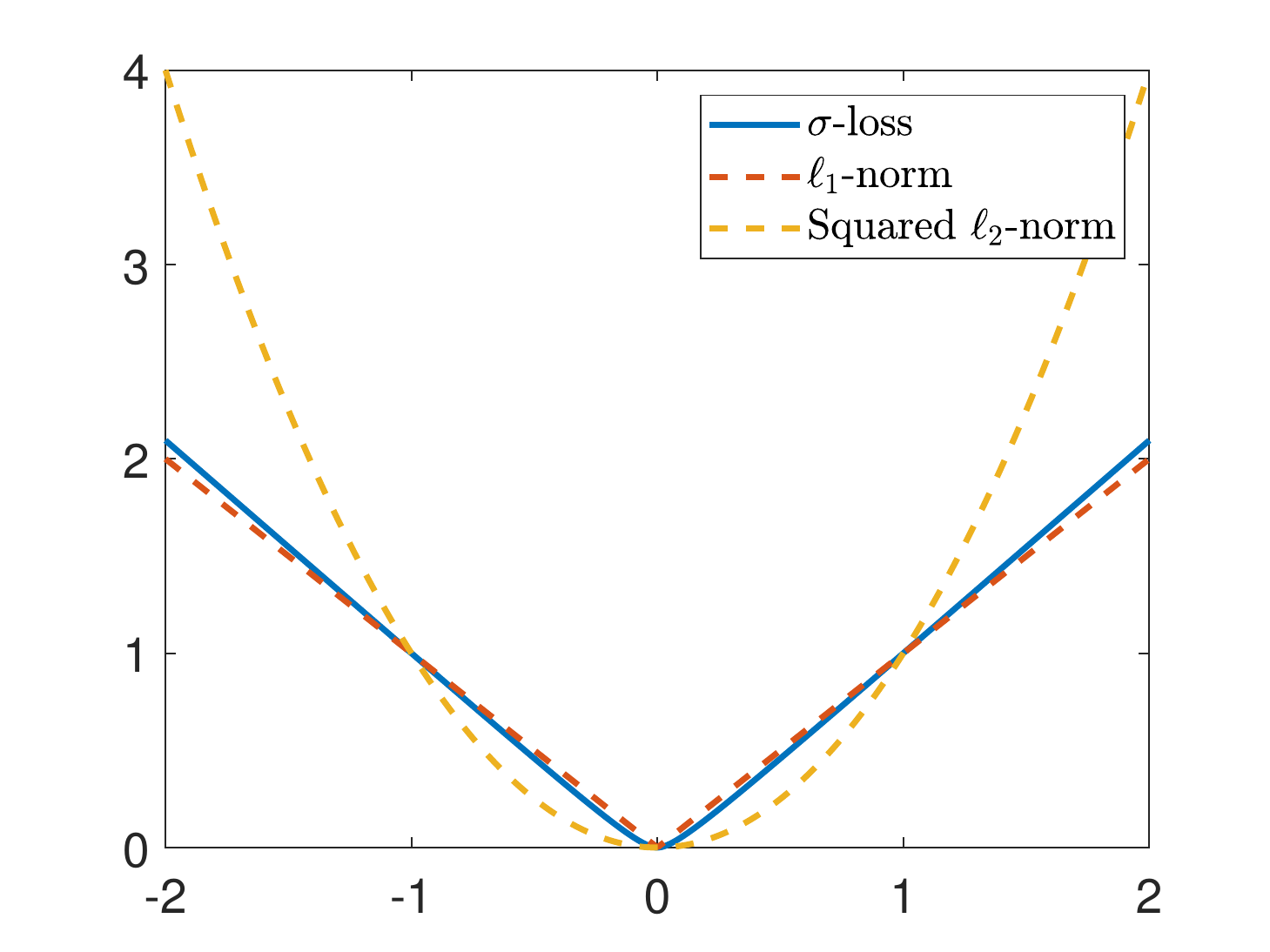}
    }
    \subfigure[$\sigma=1$]{
        \includegraphics[width=0.46\linewidth]{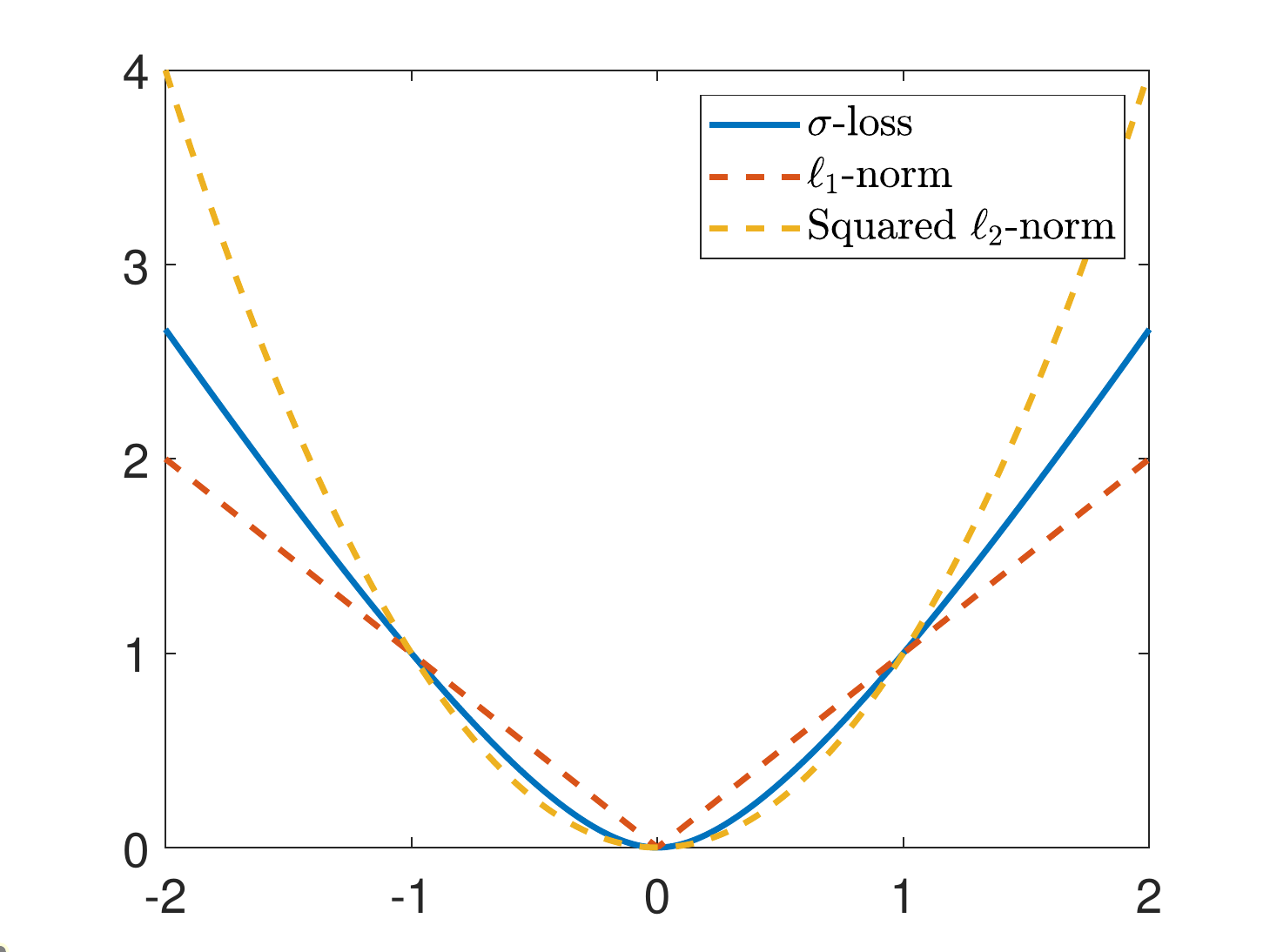}
    }

    \caption{Illustrations of the $\sigma$-loss function when 
    $A \in \mathbb R^{1 \times 1}$. Note that the squared Frobenius 
    norm of vectors becomes squared $\ell_2$-norm and the $\ell_{2,1}$-norm of vectors is $\ell_1$-norm.}
    \label{figure_adaptive}
    \vspace{-3mm}
\end{figure}

\begin{myRemark}
Assume that $\exists f_i, f_i = 0$. Let $\mathcal D = \{i \mid f_i = 0\}$ 
and $|\mathcal D|$ represents the size of $\mathcal D$. 
When $|\mathcal D| > 1$, it is obvious that $k$ can be any integers lying 
in $[2, |\mathcal D|]$ and $w_{i \notin \mathcal D} = 0$. 
In this case, we can simply set the corresponding weights 
$w_i = \frac{1}{|\mathcal D|}$. 
If $|\mathcal D| = 1$, we can simply add a tiny value to $f_{i \in \mathcal D}$
such that $k \geq 2$ holds and the trivial solution can be avoided. 
\end{myRemark}

In practical scenarios, loss of each sample is usually positive such 
that the condition of Theorem \ref{uniqueness} holds. 
The concrete procedure of optimization based on $f_i > 0$ 
is summarized in Algorithm \ref{algorithm_framework}.

\vspace{-1mm}

\subsection{Extension of Co-Robust Weight Learning}
As shown in Eq. (\ref{fusion}), the framework is equivalent to the fusion 
of weight learning under the exponential regularization. 
To increase (or decrease) the amount of the adopted items, 
the proposed co-robust weight learning can be extended as

\begin{equation}
    \begin{split}
        \min \limits_{w_i, \bm \theta} & \sum \limits_{j=1}^\infty \sum \limits_{i=1}^n w_i^{\gamma \cdot j} f(\bm x_i; \bm \theta), 
        ~~ s.t.~  \sum \limits_{i=1}^n w_i = 1, 0 \leq w_i < 1.
    \end{split}
\end{equation}
where $\gamma > 0$ is a hyper-parameter to control the density of fused items. 
If $\gamma < 1$, the density of fused items increases; if $\gamma > 1$, 
the density of adopted items becomes more sparse. 
To keep simplicity, the above formulation can be rewritten as
\begin{equation}
    \label{co_robust_framework_extension}
    \begin{split}
        \min \limits_{w_i, \bm \theta}~ & \sum \limits_{i=1}^n \frac{1}{1 - w_i^\gamma} f(\bm x_i; \bm \theta), 
        ~~ s.t. ~ \sum \limits_{i = 1}^n w_i = 1, 0 \leq w_i < 1 .
    \end{split}
\end{equation}
Furthermore, the uniqueness of $\{w_i\}_{i=1}^n$ still holds provided that 
$f_i > 0$ for any $i$. 
The proof is similar with the one of Theorem \ref{uniqueness}. 
Unfortunately, the solution of Eq. (\ref{co_robust_framework_extension}) is 
not similar with the one in Eq. (\ref{co_robust_framework}) any more. 
In this paper, we just pay more attentions on the special case when 
$\gamma = 1$, \textit{i.e.}, Eq. (\ref{co_robust_framework}). Extensive experiments 
demonstrate that even the simple setting of $\gamma$ can lead to impressive 
performance.

\subsection{An Example of Co-Robust Framework with A Scalable Robust Loss: $\sigma$-Loss}
According to the above discussions, the weight learning of the proposed framework
is parameter-free. Therefore, a robust loss with hyper-parameters to tune the 
robustness is recommended. The potential motivation to use a robust loss 
with hyper-parameters is to ensure the scalability. It should be emphasized 
that one can employ robust losses without extra parameters (\textit{e.g.}, 
$\ell_1$-norm, $\ell_{2,1}$-norm, \textit{etc.}) if a parameter-free model 
is required.

\subsubsection{$\sigma$-Loss}

Here we develop a robust loss, $\sigma$-loss, that interpolates between 
$\ell_{2,1}$-norm and squared Frobenius-norm and can be extended into 
point-wise form easily. 
The motivation of $\sigma$-loss is that the objective functions of 
many learning algorithms (\textit{e.g.}, least 
squares regression, $k$-means, PCA, \textit{etc}.) can be viewed as 
the squared Frobenius norm which is too sensitive to outliers. 
Consequently, a quantity of algorithms, like variations of PCA \cite{r1pca,rpca-om}, 
have been developed to promote the robustness of model by applying the more 
robust loss, $\ell_{2,1}$-norm.
However, the drawback of $\ell_{2,1}$-norm is the sensitivity to small
perturbations compared with squared Frobenius-norm. Contrasively,  the 
value of squared Frobenius-norm increases quadratically with the increase 
of input, but it is robust to small perturbations. 
Therefore, we attempt to design a loss that balances the two losses. 
The definition  of $\sigma$-loss for matrix $A \in \mathbb R^{m \times n}$ 
is given as follows 
\begin{equation} \label{adaptive_loss}
    \|A\|_\sigma = \sum \limits_{i=1}^m \frac{(1 + \sigma)\|\bm a_i\|_2^2}{\|\bm a_i \|_2 + \sigma} ,
\end{equation}
where $\sigma$ controls robustness to different type outliers. 
To further understand the motivation of $\sigma$-loss, we give 
some insights as follows:
\begin{itemize}
    \item [1.] $\|A\|_\sigma$ is non-negative, twice differential and convex such that it is appropriate as a loss function.
    \item [2.] When $\sigma \rightarrow 0$, then $\|A\|_\sigma \rightarrow \|A\|_{2,1}$.
    \item [3.] When $\sigma \rightarrow \infty$, then $\|A\|_\sigma \rightarrow \|A\|_F^2$.
\end{itemize}
Fig. \ref{figure_adaptive} demonstrates the effect of $\sigma$-loss 
function when $A \in \mathbb R^{1 \times 1}$. 
Clearly, $\sigma$-loss function interpolates between 
$\ell_{2,1}$-norm and squared Frobenius-norm.

\begin{myRemark}
    $\sigma$-loss is not a valid norm since $\forall k \geq 0, 
    \|k A\|_\sigma \neq k \|A\|_\sigma$ .
\end{myRemark}

\subsubsection{A Specific Example for Co-Robust Framework with Point-Wise $\sigma$-Loss}
To apply $\sigma$-loss into the proposed co-robust framework, 
Eq. (\ref{adaptive_loss}) is extended into the point-wise form. 
To simply the discussion, we define 
\begin{equation} 
	\|\bm a\|_{\hat \sigma} = \frac{(1+\sigma) \|\bm a\|_2^2}{\|\bm a\|_2 + \sigma} .
\end{equation}
Accordingly, let $h(\cdot) = \|\cdot\|_{\hat \sigma}$ and \textbf{the collaborative-robust 
framework with point-wise $\sigma$-loss} is formulated as 
\begin{equation} \label{obj_framework_with_sigma_loss}
	\min \limits_{w_i, \bm \theta} \sum \limits_{i=1}^n \frac{1}{1 - w_i} \|\bm g(\bm x_i; \bm \theta)\|_{\hat \sigma}.
\end{equation}
Therefore, $\sigma$-loss and weight learning promote the robustness collaboratively.
A vital obstacle of applying a robust loss is how to optimize it sufficiently, since the formulation is more complicated than squared Frobenius-norm. 
In the rest of this section, an efficient optimization algorithm is derived and its convergence is proved.

\begin{algorithm}[t]
    \caption{Algorithm to Solve Problem (\ref{opt_general_adaptive_loss}) } 
    \label{algorithm_adaptive}
    \begin{algorithmic}
        \REQUIRE Data vector $\bm{x}$ and $\bm s$ whose $i$-th element is $s_i$
        \REPEAT
        \STATE1. Calculate $d_i = (1+\sigma) \sum \limits_i \frac{\|\bm g(\bm{x}_i; \bm \theta)\|_2+2\sigma}{2(\| \bm g(\bm{x}_i; \bm \theta)\|_2 + \sigma)^2}$;
        \STATE2. Update $\bm{\theta}$ by solving Problem (\ref{opt_adaptive_loss_obj});
        \UNTIL {convergence}
        \ENSURE Solution $\bm{\theta}$.
    \end{algorithmic}

\end{algorithm}

\subsubsection{Optimization of Point-Wise $\sigma$-Loss}
To give a universal optimization algorithm for point-wise $\sigma$-loss, we generalize Eq. (\ref{obj_framework_with_sigma_loss}) as 
\begin{equation} \label{opt_general_adaptive_loss}
    \min \limits_{\bm \theta} \sum \limits_i s_{i}\|\bm g(\bm{x}_i; \bm \theta)\|_{\hat{\sigma}} .
\end{equation}
where $\bm g_i(\cdot)$ is a vector-output function.

Taking the derivative of Eq. (\ref{opt_general_adaptive_loss}) w.r.t. $\bm \theta$ and setting it to 0, we have
\begin{equation} \label{opt_derivative_adaptive_loss}
    2(1+\sigma)\sum \limits_i s_i\frac{\|\bm g(\bm{x}_i; \bm \theta)\|_2+2\sigma}{2(\|\bm g(\bm{x}_i; \bm \theta)\|_2+\sigma)^2}\nabla \bm g(\bm{x}_i; \bm \theta) \cdot \bm g(\bm{x}_i; \bm \theta)=0 .
\end{equation}
Denote
\begin{equation} \label{opt_definition_of_d}
    d_i = (1+\sigma)\frac{\|\bm g(\bm{x}_i; \bm \theta)\|_2 + 2\sigma}{2(\|\bm g(\bm{x}_i; \bm \theta)\|_2 + \sigma)^2} .
\end{equation}
Then Eq. (\ref{opt_derivative_adaptive_loss}) can be rewritten as
\begin{equation} \label{opt_simple_derivative_adaptive_loss}
    2\sum \limits_i s_i d_i \nabla \bm g(\bm{x}_i; \bm \theta) \cdot \bm g(\bm{x}_i; \bm \theta) = 0 .
\end{equation}
If $d_i$ is fixed, then the following problem, 
\begin{equation} \label{opt_adaptive_loss_obj}
    \min \limits_{\bm \theta} \sum \limits_i s_i d_i \|\bm g(\bm{x}_i; \bm \theta)\|_2^2 ,
\end{equation}
has the same solution with Eq. (\ref{opt_simple_derivative_adaptive_loss}). 
The algorithm to solve problem (\ref{opt_general_adaptive_loss}) by updating $d_i$ iteratively is summarized in Algorithm \ref{algorithm_adaptive}. 
The following theorem proves the convergence of Algorithm \ref{algorithm_adaptive}.

\begin{myTheo} \label{convergence}
    The Algorithm \ref{algorithm_adaptive} will monotonically decrease Eq. (\ref{opt_general_adaptive_loss}) in each iteration.
\end{myTheo}

\begin{myCorollary}
	Algorithm \ref{algorithm_adaptive} will converge into a local minimum. Furthermore, if $g(\cdot; \bm \theta)$ is convex, then Algorithm \ref{algorithm_adaptive} will converge into the optimum. 
\end{myCorollary}

\section{Enhanced Principal Component Analysis under the Co-Robust Framework}
Based on the above discussions, we propose an unsupervised dimensionality 
reduction method, namely Enhance Principal Component Analysis (EPCA), 
via the co-robust framework proposed in Section \ref{section_framework}. 

\subsection{Enhanced Principal Component Analysis}
Consider the space spanned by the data points (\textit{i.e.}, the column 
space of matrix $X$), and then we attempt to find a $c$-rank subspace 
that approximates $X$ as accurate as possible. 
The standard process of PCA and its variations is to centralize the data at 
first.
With the preprocessed data, the objective can be formulated as
\begin{equation} \label{obj_PCA}
    \min \limits_{{\rm rank}(Z) = c} \|X - Z\|_F^2 .
\end{equation}

Instead of preprocessing data mechanically, we intend to incorporate the 
preprocessing into model. In the following part, we will find that the 
co-robust framework generates a specific processing of data.

To understand our motivation, it is necessary to reconsider the aim of the 
centralization. If the data points scatter in an affine space, it is hard 
to find a well-fitting low-rank subspace. Fortunately, an affine space can be transformed
into a linear space via a proper translation, which can be formulated as
\begin{equation}
	\min \limits_{\bm m, {\rm rank}(Z) = c} \|X - \bm m \textbf{1}^T - Z\|_F^2 .
\end{equation}
where $\bm m$ represents the translation vector. 
Specially, the classical PCA and its many variations are equivalent to 
fix $\bm m$ as $\frac{1}{n} \sum_{i=1}^n\bm x_i$.

To apply the proposed co-robust framework to PCA model, Eq (\ref{obj_PCA})
can be rewritten as the point-wise formulation,
\begin{equation}
    \min \limits_{\bm m, {\rm rank}(Z) = c} \sum_{i=1}^n \|\bm x_i - \bm m \textbf{1}^T - \bm z_i\|_2^2.
\end{equation}
Combining with Eq. (\ref{obj_framework_with_sigma_loss}), the co-robust framework can be applied and the objective of 
Enhanced Principal Component Analysis (\textit{EPCA}) is formulated as
\begin{equation} 
    \min \limits_{\bm m, \bm \alpha^T \textbf{1} = 1, 0 \leq \alpha_i < 1, {\rm rank}(Z) = c} \sum \limits_{i=1}^n \frac{1}{1-\alpha_i} \|\bm x_i - \bm m - \bm z_i\|_{\hat \sigma} .
\end{equation}
As ${\rm rank}(Z) = c$, $\mathcal R(Z)$ can be spanned by $c$ orthonormal 
vectors $\{\bm w_i\}_{i=1}^c$ and $\bm z_i$ can be formulated as $\bm z = W \bm v_i$ 
where $W \in \mathbb R^{d \times c}$.
Note that $\bm v_i$ is the coordinate of the $i$-th sample under 
$\{\bm w_i\}_{i=1}^c$. 
Accordingly, the original problem can be transformed into
\begin{equation} \label{obj}
    \min \limits_{\bm m, V, \bm \alpha^T \textbf{1} = 1, 0 \leq \alpha_i < 1, W^T W = I} \sum \limits_{i=1}^n \frac{1}{1-\alpha_i} \|\bm x_i - \bm m - W \bm v_i \|_{\hat \sigma} .
\end{equation}

An attractive property of EPCA is its rotational invariance, which is unavailable
in a large proportion of PCA methods \cite{l1pca,l1pca-nongreedy,lppca}.
Here, we give the definition of the rotational invariance. 

\begin{myDefinition}
    Let $R$ be an orthogonal matrix, $X$ be the dataset, and $\mathcal{F}(X)$ denote a representation 
    learner that returns a new representation of $X$. If $\mathcal{F}(X) = \mathcal{F}(R X)$
    holds for any valid $R$ and $X$, then $\mathcal{F}(\cdot)$ is rotational invariant.
\end{myDefinition}

\begin{myTheo} \label{theorem_rotational_invariant}
    Given an orthogonal matrix $R$ subjected to $R^T R = R R^T = I$, EPCA on 
    $R X$ and $X$ will obtain the same low-rank representation, $V$. In other 
    words, EPCA is rotational invariant. 
\end{myTheo}

\begin{algorithm}[t]
    \caption{Algorithm to solve problem (\ref{obj})}
    \label{algorithm_core}
    \begin{algorithmic}
        \REQUIRE The tradeoff parameter $\sigma$, initialized $\{\alpha_i\}_{i=1}$ and dataset $X = [\bm x_1, \bm x_2, \cdots, \bm x_n]$.
        \REPEAT 
            \STATE Update $d_i$ by Eq. (\ref{d_update}).
            \STATE Update $\bm v_i$ by Eq. (\ref{v_update}). 
            \STATE Update $\bm m$ by Eq. (\ref{m_update}).
            \STATE Update $W$ by solving problem (\ref{W_obj}).
            \REPEAT
                \STATE Find a new $k$. \\
            \UNTIL{$k$ satisfies constraint (\ref{alpha_k_constraint}).}
            \STATE Update $\alpha_i$ by Eq. (\ref{alpha_update}).
        \UNTIL{convergence}
        \ENSURE Orthonormal basis $W$ of and translation vector $\bm m$.
    \end{algorithmic}
\end{algorithm}

\subsection{Optimization of EPCA}

The key part of EPCA is how to solve its objective efficiently, 
which is shown as follows:

\textbf{Evaluate $d_i$}: According to Algorithm \ref{algorithm_adaptive}, the point-wise $\sigma$-loss can be optimized by solving a weighted $\ell_2$-norm problem. Specifically speaking, $d_i$ is evaluated by
\begin{equation} \label{d_update}
	d_i = (1 + \sigma)\frac{\|\bm x_i - \bm m - W \bm v_i\|_2 + 2\sigma}{2(\|\bm x_i - \bm m - W \bm v_i\|_2 + \sigma)^2} ,
\end{equation}
and problem (\ref{obj}) is converted into 
\begin{equation} \label{obj_with_d}
	\min \limits_{\bm m, V, W^T W = I} \sum \limits_{i=1}^n \frac{d_i}{1-\alpha_i} \|\bm x_i - \bm m - W \bm v_i\|_2^2 .
\end{equation}

\textbf{Optimize $V$}: 
Let $\mathcal J_i = \frac{d_i}{1 - \alpha_i} \|\bm x_i - \bm m - W \bm v_i\|_2^2$. 
Take the derivative regarding $\bm v_i$ and set it to 0, 
\begin{equation}
	\frac{\partial \mathcal J_i}{\partial \bm v_i} = \frac{2 d_i}{1 - \alpha_i} (W^T W \bm v_i - W^T (\bm x_i - \bm m)) = 0.
\end{equation}
Therefore, the optimal $\bm v_i$ can be calculated as 
\begin{equation} \label{v_update}
	\bm v_i = W^T( \bm x_i - \bm m) .
\end{equation}

\textbf{Optimize $\bm m$}: Since $\bm v_i$ can be calculated directly 
and $\alpha_i$ is viewed as constant temporarily, 
problem (\ref{obj_with_d}) can be rewritten as  
\begin{equation}
    \label{opt_obj_fix_w}
    \min \limits_{W^T W = I, \bm m} \sum \limits_{i=1}^n \frac{d_i}{1 - \alpha_i} \|(I - WW^T)(\bm x_i - \bm m)\|_2^2 ,
\end{equation}
To keep notations uncluttered, let $\eta_i = \frac{d_i}{1-\alpha_i}$. 
Taking the derivative w.r.t. $\bm m$ and setting it to 0, we obtain
\begin{equation}
    \sum \limits_{i=1}^n 2 \eta_i (I - W W^T)^2 \bm m - 2 \eta_i (I - W W^T)^2 \bm x_i = 0 .
\end{equation}

Note that $(I - W W^T)$ is idempotent, \textit{i.e.}, 
$(I - W W^T)^2 = (I - W W^T)$. 
Let 
\begin{equation} \label{m_update}
    \bm \mu = \frac{\sum \limits_{i=1}^n \eta_i \bm x_i}{\sum \limits_{i=1}^n \eta_i},
\end{equation}
and we can derive that 
\begin{equation} \label{eq_mean}
	(I - W W^T) \bm m = (I - W W^T) \bm \mu .
\end{equation}
The following theorem gives the formulation of its solutions.
\begin{myTheo} \label{theorem_solution_m}
    The solution set of $\bm m$ can be formulated as 
    \begin{equation}
        \{\bm m | \bm m = {\bm \mu} + W \bm \beta\},
    \end{equation}
    where $\bm \beta \in \mathbb R^{c}$ denotes an arbitrary vector and will not 
    affect the value of the objective function.
\end{myTheo}

In practice, $W$ is unavailable before problem (\ref{obj_PCA}) is 
completely solved. 
Since any $\bm \beta$ results in the same objective value, 
we can simply use the particular solution (\textit{i.e.}, 
$\bm \beta = 0$), $\bm m = \bm \mu$.

\textbf{Optimize $W$}:
Based on the closed-form solution of $\bm v_i$, 
Eq. (\ref{obj_with_d}) can be rewritten as
\begin{equation}
    \begin{split}
        & \sum \limits_{i=1}^n \frac{d_i}{1 - \alpha_i} \|(I - W W^T)(\bm x_i - \bm m )\|_F^2 \\
        = ~& \sum \limits_{i=1}^n \eta_i [{\rm tr} (\bar{\bm x}_i^T \bar{\bm x}_i) - 2 {\rm tr}(\bar{\bm x}_i^T W W^T \bar{\bm x}_i) \\
            & ~~~~~~~~~~~~~ + {\rm tr}(W W^T \bar{\bm x}_i \bar{\bm x}_i^T W W^T)]\\
        = ~ & \sum \limits_{i=1}^n \eta_i [{\rm tr} (\bar{\bm x}_i^T \bar{\bm x}_i) - {\rm tr}(\bar{\bm x}_i^T W W^T \bar{\bm x}_i)].\\
    \end{split}
\end{equation}
where $\bar{\bm x}_i = \bm x_i - \bm m $. As $\bm m$ and $\alpha_i$ are fixed
as constants, problem (\ref{obj_with_d}) is equivalent to
\begin{equation}
    \label{W_obj}
        \max \limits_{W^T W = I} \sum \limits_{i=1}^n {\rm tr}(W^T Q W) ,
\end{equation}
where 
\begin{equation}
    Q = \sum \limits_{i=1}^n \eta_i \bar{\bm x}_i \bar{\bm x}_i^T = \sum \limits_{i=1}^n \eta_i (\bm x_i - \bm m) (\bm x_i - \bm m)^T .
\end{equation}
Consequently, the optimal $W$ can be obtained by calculating $c$ eigenvectors of $c$ largest eigenvalues. 

\textbf{Update $\alpha_i$}: Substitute $g_i = \|(I - WW^T)(\bm x_i - \bm m)\|_{\hat \sigma}$ into Eq. (\ref{w_update}) and we have
\begin{equation}
    \label{alpha_update}
    \alpha_i = (1 - \frac{(k - 1) \sqrt{g_i}}{\sum \limits_{j=1}^k g_i}) ,
\end{equation}
where $k$ has to satisfy 
\begin{equation}
    \label{alpha_k_constraint}
    \frac{\sum \limits_{i=1}^k \sqrt{g_i}}{\sqrt{g_{k+1}}} + 1 
    \leq k < \frac{\sum \limits_{i=1}^k \sqrt{g_i}}{\sqrt{g_{k}}} + 1 .
\end{equation}

The optimization procedure of problem (\ref{obj}) is summarized in Algorithm 
\ref{algorithm_core}. 
In our implementation, the strategy to find $k$ is to search it from 1 to 
$n$ simply. 

\section{Proofs of Theorems} \label{section_proof}
\subsection{Proof of Theorem \ref{uniqueness}}
\begin{proof}
    Without loss of generality, suppose that $0 \leq f_1 \leq f_2 \leq \cdots \leq f_n$. Let $\mathcal J = \sum \limits_{i=1}^n \frac{f_i}{1-w_i}$. As $\sum \limits_{i=1}^n w_i = 1$, we have 
    \begin{equation}
        \mathcal J = \frac{f_1}{\sum \limits_{i=2}^n w_i} + \sum \limits_{i=2}^n \frac{f_i}{1 - w_i} .
    \end{equation}
    Take derivative of $\mathcal J$ w.r.t. $w_i$
    \begin{equation}
        \frac{\partial \mathcal J}{\partial w_i} = - \frac{f_1}{(\sum \limits_{i=2}^n w_i)^2} + \frac{f_i}{(1 - w_i)^2} ,
    \end{equation}
    and furthermore, the Hessian matrix is given as
    \begin{equation}
        H_{ij} = \frac{\partial^2 \mathcal J}{\partial w_i \partial w_j} = 
        \left \{ 
        \begin{array}{l l}
            \frac{2 f_1}{(\sum \limits_{i=2}^n w_i)^3} + \frac{2 f_i}{(1 - w_i)^3}, & i = j ,\\
            \frac{2 f_1}{(\sum \limits_{i=2}^n w_i)^3}, & i \neq j .
        \end{array}
        \right.
    \end{equation}
    
    Accordingly, the Hessian matrix $H$ can be represented as 
    \begin{equation}
        H = \frac{2 f_1}{(\sum \limits_{i=2}^n w_i)^3} \textbf{1} \textbf{1}^T + {\rm diag}(\frac{2 f_i}{(1 - w_i)^3}) .
    \end{equation}
    
    It is not hard to realize that $\forall \bm x \in \mathbb R^n$, we have 
    $\bm x^T H \bm x >0$. 
    Clearly, the problem (\ref{co_robust_framework}) is convex and the 
    optimal $\bm w$ is unique. Note that the function
    \begin{equation}
        g_i(\lambda) = w_i= (1 - \sqrt{\frac{f_i}{\lambda}})_+,
    \end{equation}
    will not decrease with $\lambda$ goes larger. 
    As $k$ is at least 2 and $\sum \limits_{i=1}^n w_i = 1$, 
    for any valid $\{w_i\}_{i=1}^n$ there is only one $\lambda$. 
    Hence, $\lambda$ is unique as well. Utilize the fact
    \begin{equation}
        \sqrt{f_k} < \sqrt{\lambda} \leq \sqrt{f_{k+1}} .
    \end{equation}
    and the theorem is proved.
\end{proof}

\subsection{Proof of Theorem \ref{convergence}}
\begin{myLemma}
    \label{lemma1}
    If arbitrary $\bm x$ and $\bm y$ have the same dimension, the following inequality holds:
    \rm{
    \begin{equation}
        \begin{split}
            \frac{\|\bm{x}\|_2^2}{\|\bm{x}\|_2+\sigma} - &\frac{\|\bm{y}\|_2 + 2\sigma}{2(\|\bm{y}\|_2 + \sigma)^2}\|\bm{x}\|_2^2  \\
            &\leq \frac{\|\bm{y}\|_2^2}{\|\bm{y}\|_2 + \sigma} - \frac{\|\bm{y}\|_2 + 2\sigma}{2(\|\bm{y}\|_2 + \sigma)^2}\|\bm{y}\|_2^2 .
        \end{split}
    \end{equation}
    }
\end{myLemma}

\begin{proof} For any $\bm{x}$, $\bm{y} \in \mathbb R^d$, 
        \begin{align*} 
	& (\|\bm{x}\|_2 - \|\bm{y}\|_2)^2 (\|\bm{x}\|_2\|\bm{y}\|_2 + 2\sigma \|\bm{x}\|_2 + \sigma \|\bm{y}\|_2) \geq 0 \\
	\Rightarrow & 2\|\bm{x}\|_2^2\|\bm{y}\|_2^2 + 3 \sigma \|\bm{x}\|_2^2 \|\bm{y}\|_2 \leq \|\bm{x}\|_2\|\bm{y}\|_2 \|\bm{y}\|_2^2 + \\ 
    & \|\bm{x}\|_2 \|\bm{y}\|_2 \|\bm{x}\|_2^2 + 2\sigma \|\bm{x}\|_2 \|\bm{x}\|_2^2 + \sigma \|\bm{y}\|_2 \|\bm{y}\|_2^2 \\
    \Rightarrow & 2\|\bm{x}\|_2^2 (\|\bm{y}\|_2 + \sigma)^2 \\
	& \leq (|\bm{y}| \|\bm{y}\|_2^2 + \|\bm{y}\|_2 \|\bm{x}\|_2^2 + 2\sigma \|\bm{x}\|_2^2)(\|\bm{x}\|_2 + \sigma) \\
	\Rightarrow & \frac{\|\bm{x}\|_2^2}{\|\bm{x}\|_2 + \sigma} \leq \frac{\|\bm{y}\|_2 \|\bm{y}\|_2^2 + \|\bm{y}\|_2\|\bm{x}\|_2^2 + 2\sigma \|\bm{x}\|_2^2}{2(\|\bm{y}\| + \sigma)^2} \\
	\Rightarrow & \frac{\|\bm{x}\|_2^2}{\|\bm{x}\|_2 + \sigma} - \frac{\|\bm{y}\|_2 + 2\sigma}{2(\|\bm{y}\|_2 + \sigma)^2} \|\bm{x}\|_2^2 \leq \frac{\|\bm{y}\|_2\|\bm{y}\|_2^2}{2(\|\bm{y}\|_2 + \sigma)^2} \\
	\Rightarrow & \frac{\|\bm{x}\|_2^2}{\|\bm{x}\|_2 + \sigma} - \frac{\|\bm{y}\|_2 + 2\sigma}{2(\|\bm{y}\|_2+\sigma)^2} \|\bm{x}\|_2^2 \\
	& \leq \frac{\|\bm{y}\|_2^2}{\|\bm{y}\|_2 + \sigma} - \frac{\|\bm{y}\|_2 + 2\sigma}{2(\|\bm{y}\|_2+\sigma)^2} \|\bm{x}\|_2^2 \\
	\Rightarrow & \frac{\|\bm{x}\|_2^2}{\|\bm{x}\|_2 + \sigma} \leq \frac{\|\bm{y}\|_2^2}{\|\bm{y}\|_2 + \sigma} .
    \end{align*}
	which completes the proof.
\end{proof}

With the help of the above lemma, we give the proof of Theorem 
\ref{convergence} as follows.

\begin{proof} [Proof of Theorem \ref{convergence}]

    Assume that we fix $d_i$ and calculate $\widetilde{\bm{\theta}}$ after step 2 of Algorithm \ref{algorithm_adaptive}, then we have:
    \begin{equation}
        \sum \limits_i d_i s_i\|\bm g(\bm x_i; \widetilde{\bm{\theta}})\|_2^2 \leq \sum \limits_i d_i s_i \|\bm g(\bm{x}_i; \bm \theta)\|_2^2 .
    \end{equation}
    Due to the definition of $d_i$ according to Eq.(\ref{opt_definition_of_d}), we have:
    \begin{equation} \label{opt:inequality by x}
        \begin{split}
            & (1+\sigma) \sum \limits_i \frac{\|\bm g(\bm{x}_i; \bm \theta)\|_2 + 2\sigma}{2(\|\bm g(\bm{x}_i; \bm \theta)\|_2 + \sigma)^2} s_i \|\bm g(\bm x_i; \widetilde{\bm{\theta}})\|_2^2 \\
            \leq & (1+\sigma) \sum \limits_i \frac{\|\bm g(\bm{x}_i; \bm \theta)\|_2 + 2\sigma}{2(\|\bm g(\bm x_i; \bm \theta)\|_2 + \sigma)^2} s_i \|\bm g(\bm x_i; \bm \theta)\|_2^2 .
        \end{split}
    \end{equation}
    And based on Lemma \ref{lemma1}, we have following inequality:
    \begin{equation}\label{opt:added inequality}
        \begin{split}
            &\frac{\|\bm g(\bm x_i;\widetilde{\bm{\theta}})\|_2^2}{\|\bm g(\bm x_i; \widetilde{\bm{\theta}})\|_2 + \sigma} - \frac{\|\bm g(\bm{x}_i; \bm \theta)\|_2 + 2\sigma}{2(\|\bm g(\bm x_i; \bm \theta)\|_2 + \sigma)^2}\|\bm g(\bm x_i;\widetilde{\bm{\theta}})\|_2^2 \\
            \leq & \frac{\|\bm g(\bm{x}_i; \bm \theta)\|_2^2}{\|\bm g(\bm{x}_i; \bm \theta)\|_2 + \sigma} - \frac{\|\bm g(\bm{x}_i; \bm \theta)\|_2 + 2\sigma}{2(\|\bm g(\bm x_i; \bm \theta)\|_2 + \sigma)^2}\|\bm g(\bm{x}_i; \bm \theta)\|_2^2 .
        \end{split}
    \end{equation}
    Through multiplying $(1+\sigma) s_i$ on both sides of (\ref{opt:added inequality}) and adding to (\ref{opt:inequality by x}), we thus have
    \begin{equation} \label{opt:inequality optimal x}
        \begin{split}
            \sum \limits_i s_i \frac{(1+\sigma) \|\bm g(\bm x_i; \widetilde{\bm{\theta}})\|_2^2}{\|\bm g(\bm x_i; \widetilde{\bm{\theta}})\|_2 + \sigma} 
            \leq & \sum \limits_i s_i \frac{(1+\sigma) \|\bm g(\bm{x}_i; \bm \theta)\|_2^2}{\|\bm g(\bm{x}_i; \bm \theta)\|_2+\sigma} .
        \end{split}
    \end{equation}
    Therefore, Algorithm \ref{algorithm_adaptive} decreases Eq.(\ref{opt_general_adaptive_loss}) monotonically in each iteration.
\end{proof}

\begin{table*}[t]
    \centering
    \caption{Reconstruction Error}
    \label{table_error}
    \renewcommand\arraystretch{1.2}
    \setlength{\tabcolsep}{3mm}
    \begin{tabular}{c c c c c c c c c c c}
        \hline

        \hline
            Dataset & $c$ & PCA & L1PCA & L1PCA-NG & R1PCA & RSPCA & L$_p$PCA & PCA-OM & PCA-AN & EPCA\\
        \hline
        \hline
            \multirow{3}{*}{JAFFE $(\times 10^7)$} & 10 & 6.92 & 7.16 & 7.29 & \underline{6.91} & 8.05 & 7.05 & 6.93 & 6.96 & \textbf{6.89}\\
            & 30 & 4.53 & 4.60 & 5.06 & 4.13 & 5.49 & 4.52 & {4.11} & \underline{4.11} & \textbf{4.11}\\
            & 50 & 4.32 & 4.08 & 4.54 & 3.94 & 5.00 & 4.01 & {3.93} & \underline{3.88} & \textbf{3.83}\\
        \hline
            \multirow{3}{*}{YALE $(\times 10^7)$} & 10 & 16.39 & 16.86 & 17.74 & 16.20 & 19.04 & 16.50 & {16.20} & \underline{16.42} & \textbf{16.19}\\
            & 30 & 10.67 & 10.77 & 11.73 & \underline{9.43} & 13.76 & 10.52 & 9.45 & 9.51 & \textbf{9.36}\\
            & 50 & 9.92 & {9.43} & 10.13 & 9.64 & 12.88 & \underline{9.29} & 9.65 & \textbf{9.14} & {9.60}\\
        \hline
            \multirow{3}{*}{ORL $(\times 10^8)$} & 10 & 1.81 & 1.82 & 1.88 & 1.71 & 2.09 & 1.81 & {1.70} & \underline{1.69} & \textbf{1.68}\\
            & 30 & 1.39 & 1.29 & 1.42 & \underline{1.11} & 1.59 & 1.27 & {1.11} & \underline{1.13} & \textbf{1.11}\\
            & 50 & 1.44 & 1.23 & 1.36 & 1.13 & 1.58 & 1.25 & \underline{1.12} & 1.17 & \textbf{1.11}\\
        \hline
            \multirow{3}{*}{COIL20 $(\times 10^9)$} & 10 & \underline{2.49} & 2.53 & 2.61 & 2.51 & 2.94 & 2.51 & 2.51 & 2.52 & \textbf{2.48}\\
            & 30 & 1.60 & 1.63 & 1.68 & {1.57} & 2.27 & 1.62 & \underline{1.57} & 1.58 & \textbf{1.56}\\
            & 50 & 1.35 & 1.38 & 1.43 & {1.32} & 2.21 & 1.36 & 1.32 & \textbf{1.30} & \underline{1.32}\\
        \hline
            \multirow{3}{*}{UMIST $(\times 10^8)$} & 10 & 7.65 & 7.93 & 8.17 & 7.34 & 8.44 & 7.82 & {7.33} & \underline{7.34} & \textbf{7.32}\\
            & 30 & 7.92 & 7.75 & 8.06 & {6.77} & 8.15 & 7.70 & 6.78 & \textbf{6.28} & \underline{6.76}\\
            & 50 & 8.93 & 8.27 & 8.66 & 7.64 & 8.86 & 8.39 & \underline{7.61} & \textbf{7.19} & \underline{7.60}\\
        \hline

        \hline
    \end{tabular}
\end{table*}

\begin{table}[t]
    \centering
    \renewcommand\arraystretch{1.4}
    \setlength{\tabcolsep}{3mm}
    \caption{Information of Datasets}
    \label{datasets}
    \begin{tabular} {c c c c}
        \hline

        \hline
        Dataset & Size & Dimensionality & Class \\
        \hline
        \hline
        JAFFE & 213 & 1024 & 10 \\

        YALE & 165 & 1024 & 15 \\

        ORL & 400 & 1024 & 40 \\

        COIL20 & 1440 & 1024 & 20 \\

        UMIST & 575 & 1024 & 20 \\
        \hline

        \hline
    \end{tabular}
\end{table}

\begin{table*}[t]
    \centering
    \caption{Clustering Accuracy(\%)}
    \label{table_acc}
    \renewcommand\arraystretch{1.2}
    \setlength{\tabcolsep}{3mm}
    \begin{tabular}{c c c c c c c c c c c}
        \hline
        
        \hline
            Dataset & $c$ & PCA & L1PCA & L1PCA-NG & R1PCA & RSPCA & L$_p$PCA & PCA-OM & PCA-AN & EPCA \\
        \hline
        \hline
            \multirow{3}{*}{JAFFE} & 10 & 86.85 & 83.57 & 77.72 & 79.81 & 69.01 & \underline{86.85} & 83.57 & 80.28 & \textbf{90.14} \\ 
            & 30 & \underline{94.37} & 86.38 & 74.23 & 90.14 & 81.22 & 92.02 & 89.67 & 83.57 & \textbf{97.65} \\ 
            & 50 & 88.73 & 83.57 & 73.57 & 90.14 & 73.24 & 77.46 & {91.08} & \underline{94.37} & \textbf{96.24} \\ 
        \hline
            \multirow{3}{*}{YALE} & 10 & \underline{43.03} & 41.82 & 38.65 & 42.42 & 40.61 & 40.00 & 37.58 & 40.61 & \textbf{46.06} \\ 
            & 30 & 38.79 & 41.21 & 38.51 & \underline{44.85} & 40.00 & 43.03 & 36.36 & 38.79 & \textbf{46.67} \\ 
            & 50 & 33.94 & 38.79 & 38.92 & \underline{43.64} & 41.82 & 38.79 & 41.21 & 43.03 & \textbf{44.24} \\ 
        \hline 
            \multirow{3}{*}{ORL} & 10 & 51.50 & 51.00 & 48.88 & \underline{52.75} & 43.25 & 50.50 & 50.00 & 52.50 & \textbf{55.75} \\ 
            & 30 & 53.75 & 55.75 & 53.09 & 49.75 & 49.75 & 55.50 & \underline{59.25} & 56.50 & \textbf{62.00} \\ 
            & 50 & 47.25 & 53.25 & 53.07 & 52.00 & 48.75 & {55.00} & 54.00 & \underline{57.00} & \textbf{57.25} \\ 
        \hline
            \multirow{3}{*}{COIL20} & 10 & 62.08 & \underline{63.89} & 61.88 & 59.65 & 62.78 & 60.97 & 60.90 & 61.94 & \textbf{64.38} \\
            & 30 & {63.61} & 59.51 & 61.32 & 58.13 & 58.13 & 63.26 & 62.15 & \underline{63.68} & \textbf{65.14} \\
            & 50 & \underline{65.56} & 60.62 & 59.98 & 63.75 & 61.88 & 60.83 & 65.28 & 63.96 & \textbf{66.39} \\ 

        \hline
            \multirow{3}{*}{UMIST}  
            & 10 & 28.00 & 29.74 & 30.85 & 28.52 & \underline{30.43} & 28.00 & 28.70 & 29.22 & \textbf{31.13} \\ 
            & 30 & \underline{31.83} & 29.57 & 30.15 & 30.26 & 30.43 & 31.13 & 30.26 & 27.89 & \textbf{33.39} \\ 
            & 50 & 29.91 & \underline{31.48} & 30.06 & 30.09 & 30.26 & 29.22 & 29.74 & 29.74 & \textbf{32.17} \\ 
        \hline
        
        \hline
    \end{tabular}
\end{table*}

\subsection{Proof of Theorem \ref{theorem_rotational_invariant}}
\begin{proof}
    If $R \in \mathbb{R}^{d \times d}$ is an orthogonal matrix, then for any 
    vector $\bm \beta \in \mathbb{R}^d$,
    \begin{equation}
        \begin{split}
            & \|R \bm x\|_2^2 = {\rm tr} (\bm x^T R ^T R \bm x) = {\rm tr}(\bm x^T x) = \|\bm x\|_2^2 \\
        \end{split}
    \end{equation}
    Accordingly, we have 
    \begin{equation}
        \begin{split}
            & \sum \limits_{i=1}^n \frac{1}{1-\alpha_i} \|\bm x_i - \bm m - W \bm v_i \|_{\hat \sigma}, \\
            = & \sum \limits_{i=1}^n \frac{1}{1-\alpha_i} \cdot \frac{(1+\sigma)\|\bm x_i - \bm m - W \bm v_i\|_2^2}{\|\bm x_i - \bm m - W \bm v_i \|_2 + \sigma} \\
            = & \sum \limits_{i=1}^n \frac{1}{1-\alpha_i} \|R \bm x_i - R \bm m - R W \bm v_i \|_{\hat \sigma} .
        \end{split}
    \end{equation}
    Therefore, after a rotational operation for $X$, the optimal low-dimensional 
    representation, $V$, of the original features is also the optimal solution
    of the transformed features. 
\end{proof}

\subsection{Proof of Theorem \ref{theorem_solution_m}}
\begin{proof}
	Clearly, $\bm \mu$ is a particular solution for problem (\ref{obj_PCA}). Let $W^\bot \in \mathbb R^{d \times (d - c)}$ represent orthogonal complement of $\{\bm w_i\}_{i=1}^c$ and $W_+ = [W; W^\bot]$. Therefore, we have
	\begin{equation}
		W W^T = W_+ 
		\left [
		\begin{array}{c c}
			I & 0 \\
			0 & 0
		\end{array}
		\right ]
		W_+^T .
	\end{equation}
	Accordingly, the eigenvalues of $(I - W W^T)$ are all 0 and 1. More specifically, $c$ of them are 0, which means the rank of the nullspace of $(I - W W^T)$. Note that 
	\begin{equation}
		(I - W W^T) W = 0.
	\end{equation} 
	Therefore, the feasible solution of $\bm m$ in Eq. (\ref{eq_mean}) can be denoted by $\bm \mu + W \bm \beta$. 
\end{proof}

\section{Experiment}
In this section, the performance of our algorithms from both reconstruction errors and clustering accuracy obtained from reconstructed data is reported.

\begin{figure*}[t]
    \centering
    \subfigure[JAFFE]{
        \includegraphics[width=0.23\linewidth]{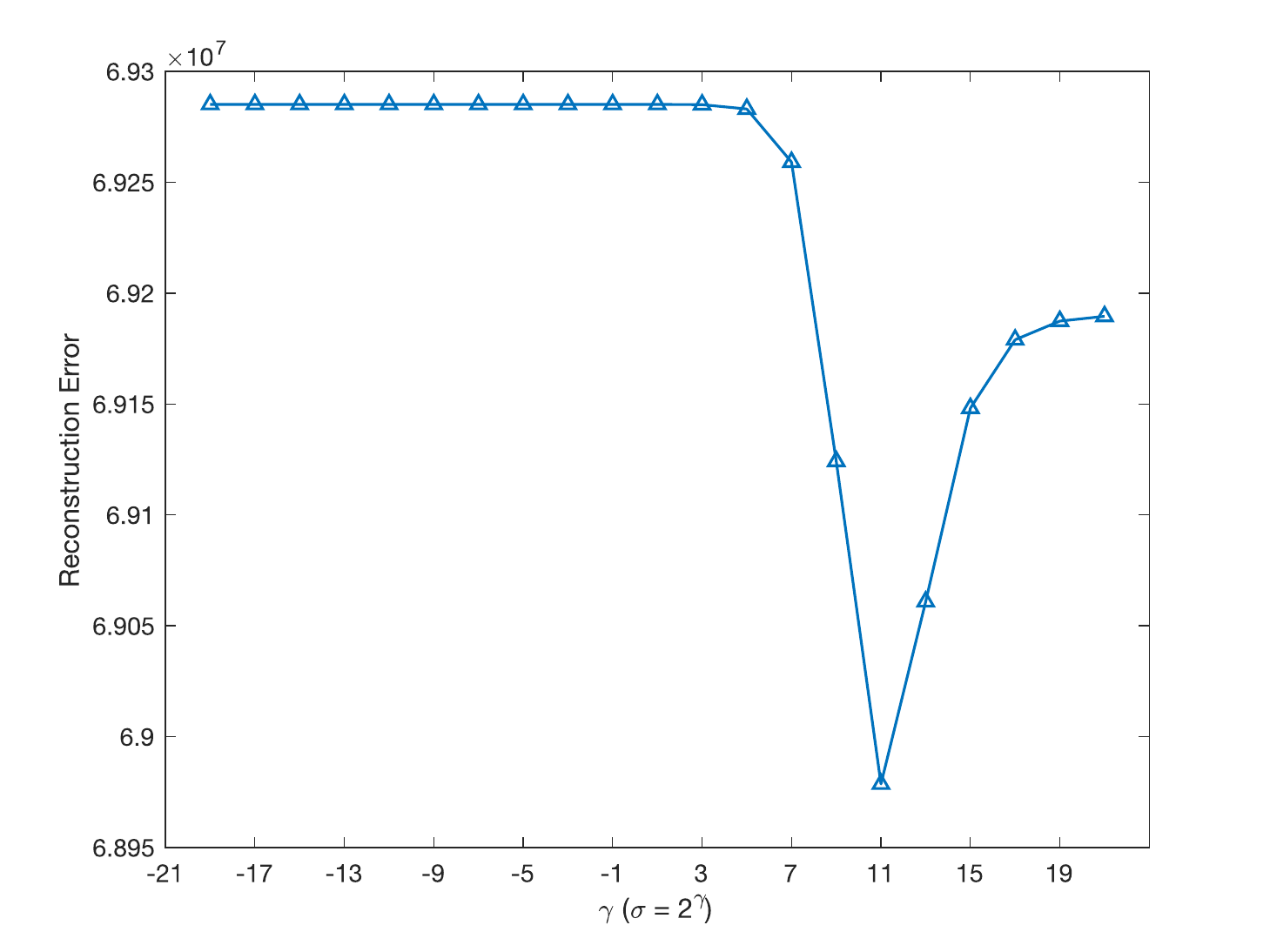}
    }
    \subfigure[YALE]{
        \includegraphics[width=0.23\linewidth]{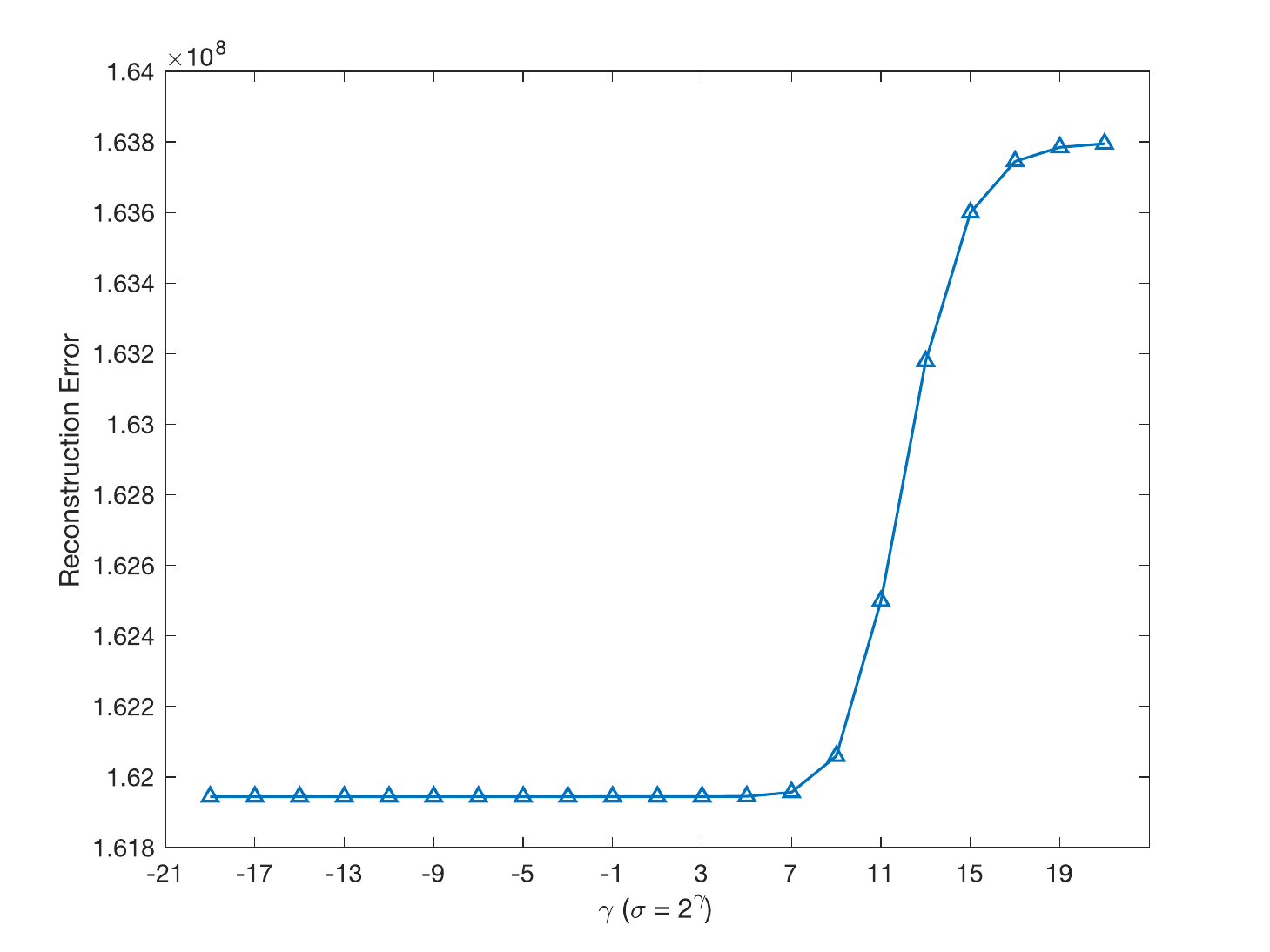}
    }
    \subfigure[ORL]{
        \includegraphics[width=0.23\linewidth]{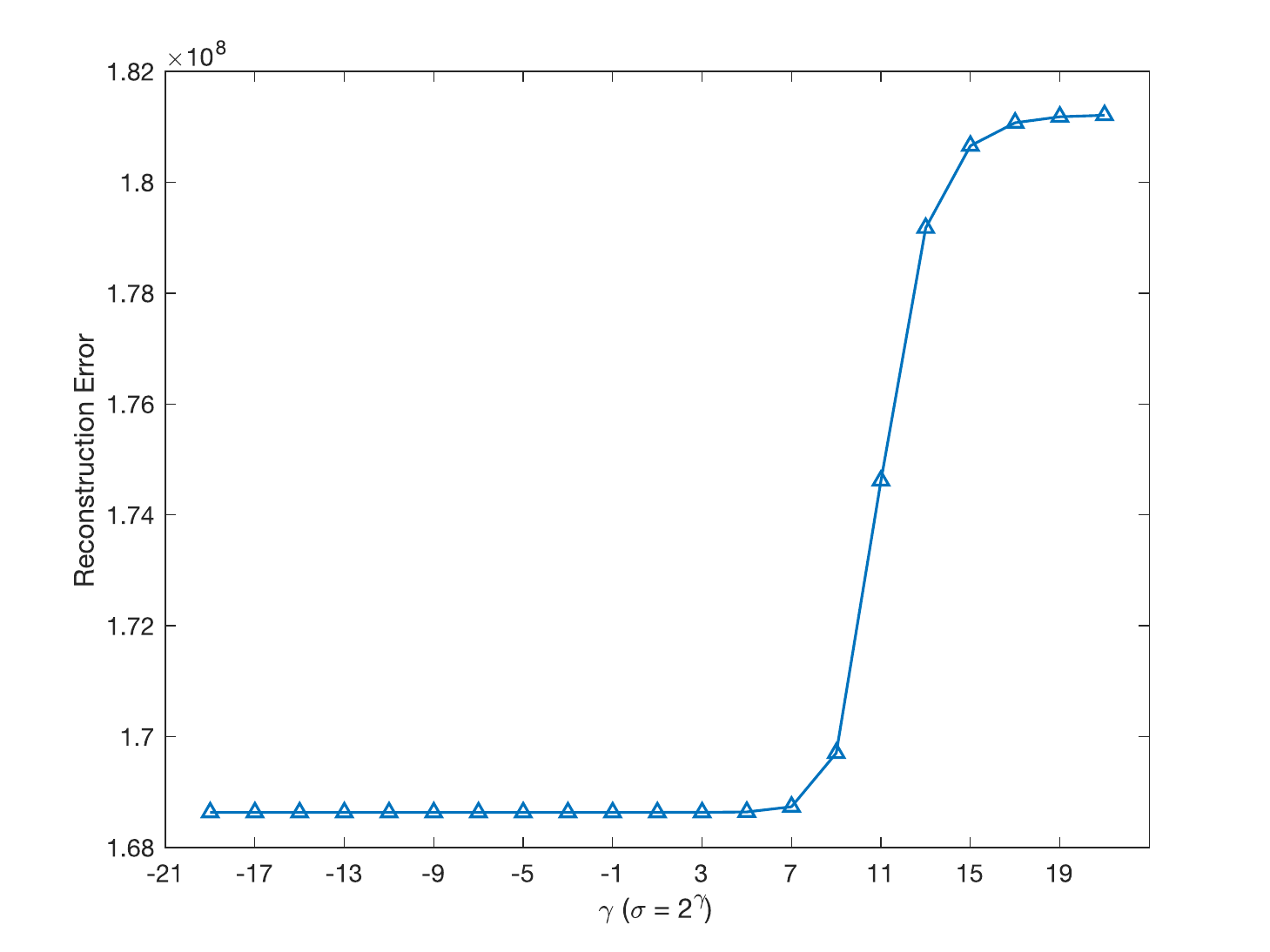}
    }
    \subfigure[COIL20]{
        \includegraphics[width=0.23\linewidth]{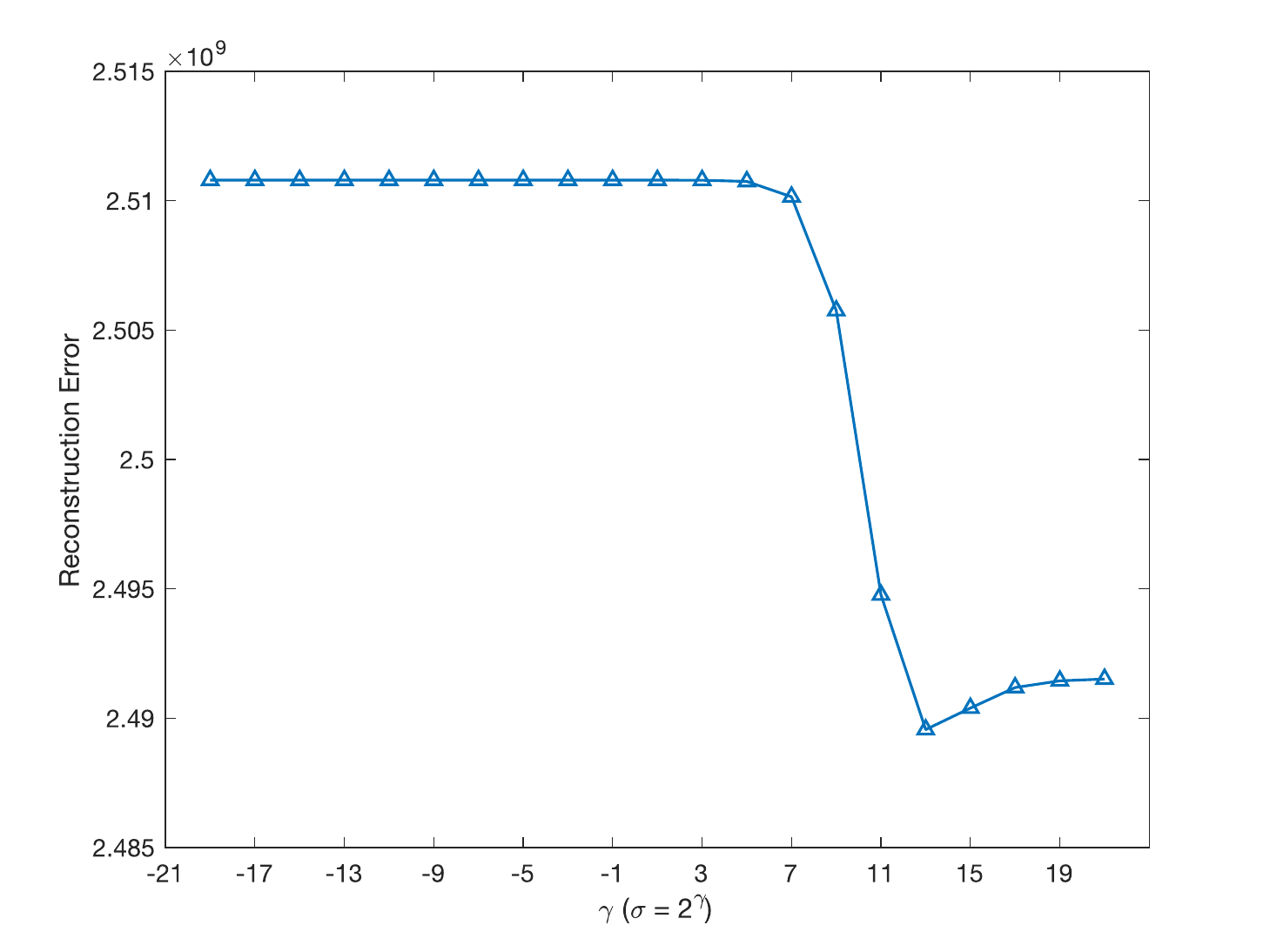}
    }
    \caption{Reconstruction errors on 4 different datasets when the rank of reconstruction matrix is fixed as 10. Then $x$-axis represents $\gamma$ where $\sigma = 2^\gamma$.}
    \label{figure_sigma}
\end{figure*}

\subsection{Experimental Setup}
The performance of proposed model is evaluated on five datasets including: JAFFE \cite{jaffe}, YALE \cite{yale}, ORL \cite{orl}, COIL20 \cite{COIL20} and UMIST \cite{umist}. All datasets are utilized to measure reconstruction errors and test clustering accuracy. The concrete information of these datasets are summarized in Table \ref{datasets}.

To demonstrate the superiority of our model, our model is compared with 8 
high related approaches including classical PCA \cite{pca}, robust PCA 
using $\ell_1$-norm (L1PCA), robust PCA using $\ell_p$-norm (L$p$PCA) \cite{lppca}, 
robust PCA using $\ell_1$-norm with non-greedy strategy 
(L1PCA-NG) \cite{l1pca-nongreedy}, robust PCA using $\ell_{2,1}$-norm 
(R1PCA) \cite{r1pca}, robust and sparse PCA (RSPCA) \cite{rspca}, robust 
PCA with optimal mean (PCA-OM) \cite{rpca-om}, and our conference work 
(PCA-AN) \cite{PCA-AN}. 
To ensure the fairness, the parameters of all baseline algorithms are 
determined via the similar approaches recorded in corresponding papers. 
In particular, we empirically fix $p = 1.2$, which can be viewed as a 
trade-off loss between $\ell_1$-norm and $\ell_2$-norm in L$p$-PCA.

To test the ability to reconstruct data of various algorithms, all datasets 
are polluted by the following strategy: 
(1) sample 20\% data points randomly from all datasets; 
(2) choose 20\% features randomly and then reset them by random values. 
All algorithms use the occluded data as input to train projection matrix $W$ and mean vector $m$.
All codes are implemented by Matlab R2019a.

\subsection{Reconstruction Errors}
The reconstruction errors are calculated by 
\begin{equation}
    \varepsilon = \|X - \bm m \textbf{1}^T - W W^T (X^{(occ)} - \bm m \textbf{1}^T)\|_F^2 ,
\end{equation}
where $W$ and $m$ are learned from occluded data, $X$ is the unpolluted data, 
and $X^{(occ)}$ is the occluded data. 
Note that in practical applications, 
the data to process can be viewed as $X^{(occ)}$ 
while $X$ is regarded as the underlying and unknown data without noises. 

Table \ref{table_error} illustrates results of different $c$ on all datasets. 
From Table \ref{table_error}, we conclude that:
\begin{itemize}
    \item [1.] EPCA obtains impressive results for different $c$ on almost all datasets as we expect. 
                PCA-AN also obtains good results on several datasets, which 
                indicates the significance of weight learning for PCA. 
    \item [2.] PCA-OM acts as a strong baseline and usually has less errors 
    than other algorithms since it applies the optimal mean and $\ell_{2,1}$-norm. 
    However, EPCA frequently obtains better performance due to utilization 
    of robust loss and augmentation to weights of better-fitting samples. 
    For instance, the error given by EPCA is $10^6$ less than the error 
    of PCA-OM on JAFFE when $c = 50$.
    \item [3.] L$p$PCA performs better than L1PCA owing to its rotational 
    invariance and appropriate $p$ but poorer than EPCA owing to the lack 
    of weight learning and non-optimal mean using.
\end{itemize}

\subsection{Clustering Accuracy}
To show the effectiveness of EPCA, the clustering accuracy (ACC) is tested on 
reconstructed data. 
Table \ref{table_acc} reports the clustering results on all datasets. 
Note that we perform $k$-means 100 times and record the mean on each dataset. 
Since the datasets are occluded severely, the superiority of EPCA is shown 
impressively. It should be pointed out that PCA-AN (our conference work) 
performs well on reconstruction errors but fails on the clustering task. 
The main reason may be that the sparse weight learning leads to the loss of 
structure information.

\begin{figure}
    \centering
    \subfigure[YALE] {
        \includegraphics[width=0.46\linewidth]{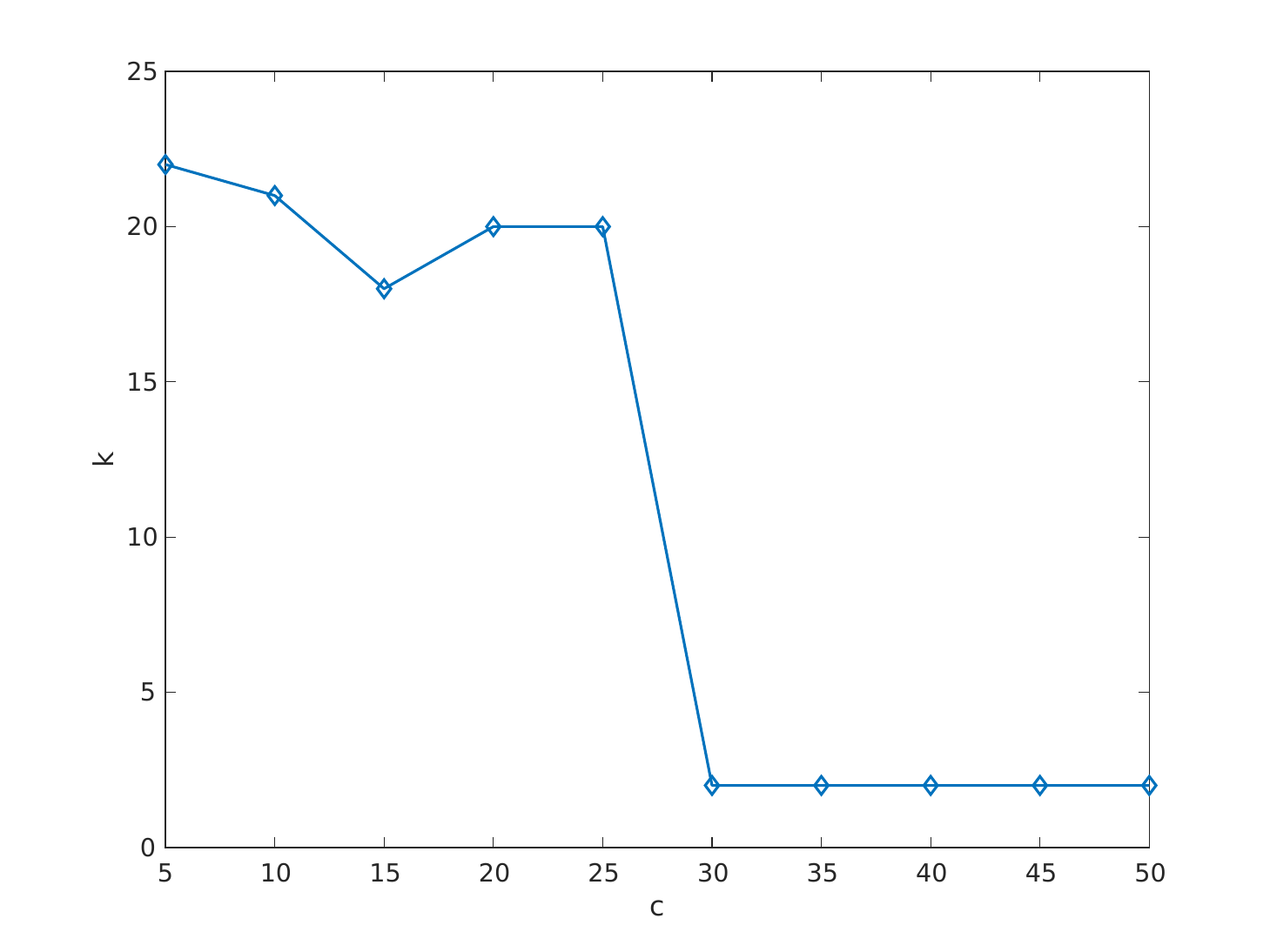}
    }
    \subfigure[ORL] {
        \includegraphics[width=0.46\linewidth]{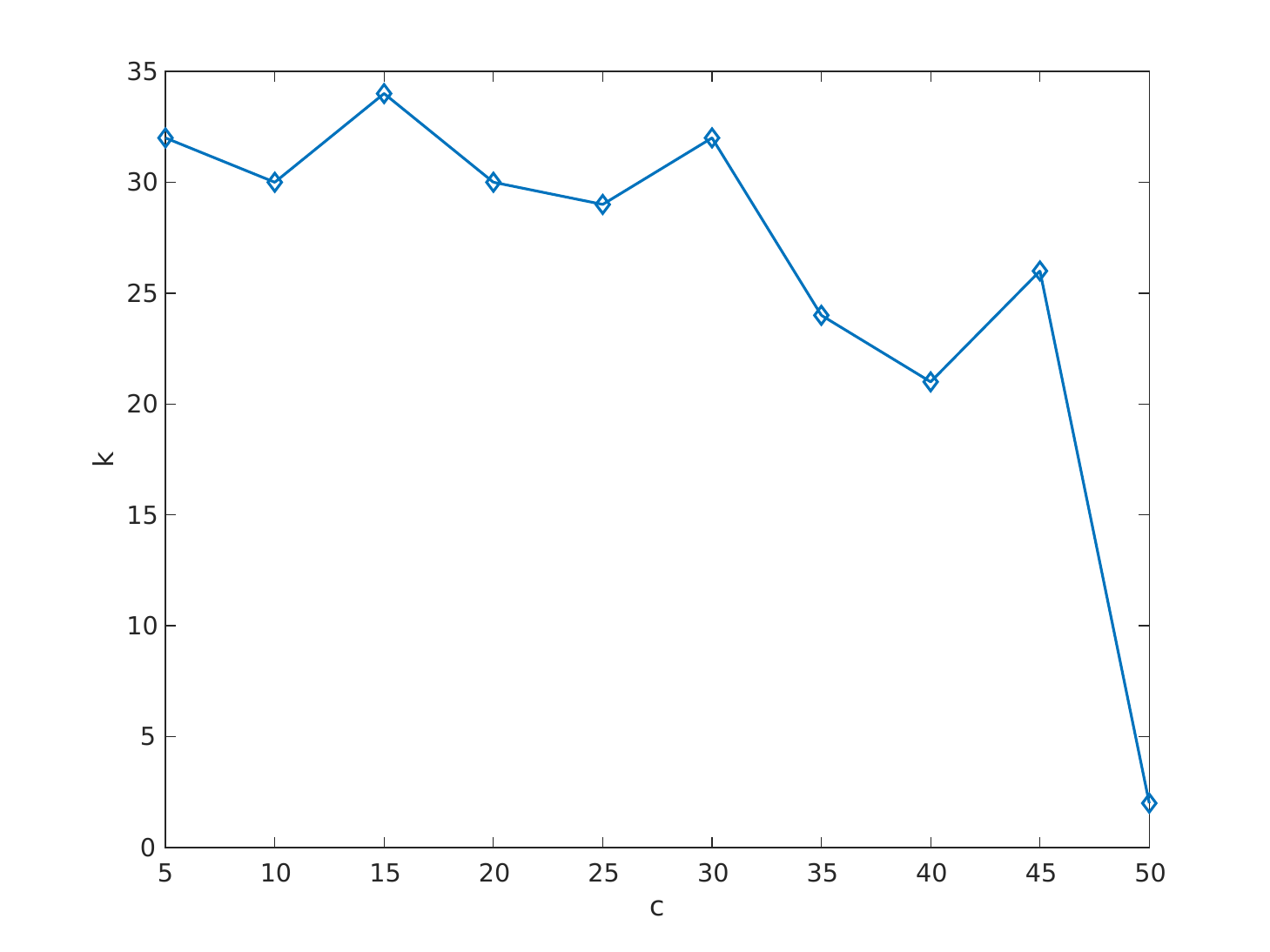}
    }
    \caption{The relationship between the number of activated samples $k$ and the rank of reconstructed matrix $c$.}
    \label{figure_k}
\end{figure}

\subsection{Parameter Sensitivity and Activated Samples}
In our experiments, the parameter $\sigma$ is determined by the following strategy: 
(1) search in a wide enough range $[2^{-20}, 2^{-19}, \cdots, 2^{20}]$; 
(2) fine tune in a narrow scope. 
Fig. \ref{algorithm_adaptive} illustrates the influence of different 
$\sigma$ when $c = 10$. 
To show the impact of $\sigma$ more intuitively, we only show the coarse 
search instead of fine tune. 
Note that the $x$-axis represents $\gamma$ instead of $\sigma$ where 
$\sigma = 2^{\gamma}$. 
From Fig. \ref{figure_sigma}, we can conclude that even when the optimal 
$\sigma$ is tiny, \textit{i.e.}, equivalent to $\ell_{2,1}$-norm, EPCA 
performs better than PCA-OM thanks to the weights $\{\alpha_i\}_{i=1}^n$. 
Additionally, Fig. \ref{figure_k} illustrates the relationship between 
the number of activated samples (denoted by $k$) and the rank of 
reconstructed matrix (denoted by $c$).

\section{Conclusion}
In this paper, based on our previous work, we first propose a general 
collaborative-robust framework of weight learning which integrates 
adaptive weight learning and robust functions 
utilization via a non-trivial approach. 
Furthermore, only $k$ samples are activated, \textit{i.e.}, the 
corresponding weights are non-zero. 
The activated samples are augmented by weight learning while the effects 
of inactivated ones are alleviated by the employed robust loss. 
Then an enhanced principal principal component analysis (EPCA) which 
extends from the proposed framework and applies the adaptive loss function 
is proposed. Moreover, the EPCA is rotational invariant. 
Extensive experiments report the superiority of our model from two aspects 
including reconstruction errors and clustering accuracy. In particular, 
the experiments on clustering exhibit the improvement of EPCA compared with 
our conference paper.

\bibliographystyle{IEEEbib.bst}
\bibliography{pca}

\end{document}